\pdfoutput=1

\documentclass[11pt]{article}

\usepackage[]{emnlp2021}
\usepackage{microtype}

\usepackage{xspace,mfirstuc,tabulary}

\makeatletter
\def\Cline#1#2{\@Cline#1#2\@nil}
\def\@Cline#1-#2#3\@nil{%
  \omit
  \@multicnt#1%
  \advance\@multispan\m@ne
  \ifnum\@multicnt=\@ne\@firstofone{&\omit}\fi
  \@multicnt#2%
  \advance\@multicnt-#1%
  \advance\@multispan\@ne
  \leaders\hrule\@height#3\hfill
  \cr}
\makeatother

\usepackage{times}
\usepackage{latexsym}

\usepackage{amsmath}
\usepackage{stmaryrd}
\usepackage{amsfonts}
\usepackage{xspace}
\usepackage{nicefrac}
\usepackage{xfrac}
\usepackage{booktabs}
\usepackage{subcaption}
\usepackage{cancel}
\usepackage{makecell}
\usepackage{adjustbox}
\usepackage{color}
\usepackage{amsthm}
\usepackage{multirow}
\usepackage{amsmath}
\usepackage{bbm}
\usepackage{enumitem}
\usepackage{float}

\usepackage{thmtools, thm-restate}

\usepackage{cleveref}
\crefname{section}{\S}{\S\S}
\Crefname{section}{\S}{\S\S}
\crefformat{section}{\S#2#1#3}
\crefname{algorithm}{Alg}{}
\crefname{algorithm}{Alg}{}
\crefname{line}{Line}{}
\crefname{equation}{}{}
\Crefname{equation}{}{}

\usepackage{emoji}
\usepackage{scalerel}
\usepackage{algorithm}
\usepackage{colortbl}
\usepackage{tabstackengine}

\usepackage[noend]{algpseudocode}
\algnewcommand{\parState}[1]{\State%
    \parbox[t]{\dimexpr\linewidth-\algmargin}{\strut\hangindent=\algorithmicindent \hangafter=1 #1\strut}}

\algrenewcommand\algorithmicindent{1.0em}%

\newcommand{\ra}[1]{\renewcommand{\arraystretch}{#1}}

\newcommand{\rightcomment}[1]{{\color{gray} \(\triangleright\) {\footnotesize\textit{#1}}}}
\algrenewcommand{\algorithmiccomment}[1]{\hfill \rightcomment{#1}}  %
\algnewcommand{\LineComment}[1]{\State \rightcomment{#1}}
\algnewcommand{\LinesComment}[1]{\State \rightcomment{\parbox[t]{\linewidth-\leftmargin-\widthof{\(\triangleright\) }}{#1}}}

\newcommand{\algorithmicfunc}[1]{\textbf{def} #1 :}
\algdef{SE}[FUNC]{Func}{EndFunc}[1]{\algorithmicfunc{#1}}{}
\makeatletter
\ifthenelse{\equal{\ALG@noend}{t}}%
  {\algtext*{EndFunc}}
  {}%
\makeatother

\theoremstyle{plain}
\newtheorem{thm}{Theorem}[section]

\newtheorem{prop}[thm]{Proposition}

\theoremstyle{definition}

\usepackage{tikz}
\DeclareMathOperator*{\argmax}{argmax}

\definecolor{darkgrey}{rgb}{0.2,0.2,0.2}
\definecolor{grey}{rgb}{0.9,0.9,0.9}
\definecolor{darkblue}{rgb}{0.0,0.0,0.5}
\definecolor{darkred}{rgb}{0.5,0.0,0.0}
\definecolor{darkorange}{rgb}{1.0,0.55,0.0}
\definecolor{darkgreen}{rgb}{0.0,0.6,0.0}
\definecolor{darkyellow}{rgb}{1.0,0.65,0.0}
\definecolor{darkorange}{rgb}{1.0,0.65,0.0}
\definecolor{darkergreen}{rgb}{0.0,0.4,0.0}
\definecolor{lightblue}{rgb}{0.8,0.8,1.0}
\definecolor{lightgreen}{rgb}{0.8,1.0,0.8}
\definecolor{lightred}{rgb}{1.0,0.8,0.8}
\definecolor{lightyellow}{rgb}{1.0,1.0,0.8}
\definecolor{lightorange}{rgb}{1.0,0.9,0.8}
\definecolor{lightgrey}{rgb}{0.96,0.97,0.98}
\definecolor{brilliantlavender}{rgb}{0.96, 0.73, 1.0}
\definecolor{coral}{rgb}{0.94,0.5,0.5}
\definecolor{yellow1}{HTML}{ffff99}
\definecolor{yellow2}{HTML}{ffff66}
\definecolor{blue1}{HTML}{CDEEFF}
\definecolor{blue2}{HTML}{97C1D3}
\definecolor{green1}{HTML}{b8e6b8}
\definecolor{green2}{HTML}{A5D260}
\definecolor{green3}{HTML}{8BC0AC}
\definecolor{green4}{HTML}{5CAD5C}

\definecolor{ryanred}{rgb}{0.64, 0.0, 0.0}
\definecolor{ryanblue}{rgb}{0.13, 0.0, 0.58}
\definecolor{ryangreen}{rgb}{0.12, 0.59, 0.0}
\definecolor{ryanpurple}{rgb}{0.65, 0.0, 0.57}

\usepackage{microtype}

\usepackage{empheq}
\usepackage[most]{tcolorbox}
\tcbuselibrary{theorems}
\tcbuselibrary{skins}

\usepackage{amssymb}

\crefname{section}{\S}{\S\S}
\Crefname{section}{\S}{\S\S}
\crefname{table}{Table}{Tables}
\crefname{figure}{Fig.}{Figs.}
\crefname{algorithm}{Alg.}{Algs.}
\crefname{algorithmic}{Alg.}{Algs.}
\crefname{appendix}{App.}{}
\crefname{lemma}{Lemma}{}
\Crefname{theorem}{Theorem}{}
\crefname{prop}{Proposition}{}
\crefname{cor}{Corollary}{}
\crefname{align}{Eq.}{}
\crefname{equation}{Eq.}{}

\newcommand{\bleu}{\textsc{bleu}\xspace}
\newcommand{\defn}[1]{\textbf{#1}}
\newcommand{\xx}{\mathbf{x}}
\newcommand{\yy}{\mathbf{y}}

\newcommand{\calB}{B}

\newcommand{\vocab}{V}
\newcommand{\pluseq}{\ \texttt{+=}\ }
\newcommand{\inc}{\pi}
\newcommand{\func}{{f}}
\newcommand{\yt}{\yy_{\scaleto{\leq t}{5pt}}}

\newcommand{\yti}{\yy^{\scaleto{(n)}{6pt}}_{\scaleto{\leq t}{5pt}}}
\newcommand{\appropto}{\mathrel{\vcenter{
  \offinterlineskip\halign{\hfil$##$\cr
    \propto\cr\noalign{\kern2pt}\sim\cr\noalign{\kern-2pt}}}}}
\DeclareMathOperator*{\E}{\mathbb{E}}

\newcommand{\asvar}{\mathbb{V}_{\mathrm{a}}}

\usepackage{accents}
\newcommand*{\adj}[1]{%
  \accentset{\mbox{\large\bfseries .}}{#1}}
\newif\iflong
\longfalse %
\newcommand{\inLongVersion}[1]{\iflong #1\fi}

\newcommand{\poisson}{P}
\newcommand{\tildepoisson}{\widetilde{P}}
\newcommand{\Q}{Q_t}
\newcommand{\tildeQ}{\widetilde{Q}_t}
\newcommand{\pipoisson}{ \pi_{\scaleto{\poisson}{4pt}}}
\newcommand{\pipoissonmc}{\widehat{\pi}^{\scaleto{\textsc{mc}}{3pt}}_{\scaleto{\poisson}{4pt}}}
\newcommand{\pipoissonis}{\widehat{\pi}^{\scaleto{\textsc{is}}{3pt}}_{\scaleto{\poisson}{4pt}}}

\newcommand{\hY}[1]{\widetilde{Y}_{\!#1}}

\newcommand{\picp}{\pi}
\newcommand{\piQ}{\pi_{\scaleto{\Q}{5pt}}}

\newcommand{\Weight}[2]{\mathrm{W}{#1 \choose #2}}
\newcommand{\bigo}[1]{\mathcal{O}\!\left(#1\right)}

\usepackage{relsize}
\newcommand{\defeq}[0]{\mathrel{\stackrel{\textnormal{\tiny def}}{=}}}
\newcommand{\defpropto}[0]{\mathrel{\stackrel{\textnormal{\tiny def}}{\propto}}}
\newcommand{\expect}[2]{\underset{#1}{\mathbb{E}}\left[#2\right]}

\usepackage[hang,flushmargin]{footmisc}

\newcommand{\ptheta}{p_{\vtheta}}
\renewcommand{\ptheta}{p}

\newcommand{\calY}{\mathcal{Y}}

\newcommand{\nmax}{T}

\newcommand{\Z}{\mathrm{Z}}

\newcommand{\bos}{\textsc{bos}\xspace}
\newcommand{\eos}{\textsc{eos}\xspace}

\usepackage{inconsolata}

\usepackage{todonotes}
\makeatletter
\newcommand*\iftodonotes{\if@todonotes@disabled\expandafter\@secondoftwo\else\expandafter\@firstoftwo\fi}  %
\makeatother

\newcommand{\ucambridge}{\emoji[openmoji]{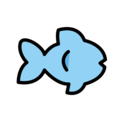}}
\newcommand{\ethz}{\emoji[openmoji]{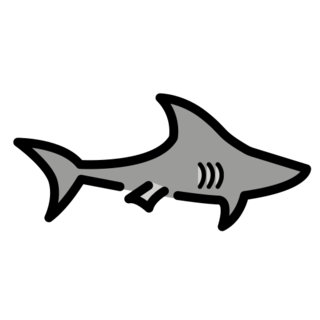}}
\newcommand{\jhu}{\emoji[openmoji]{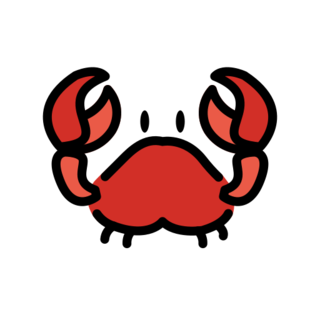}}

\newcommand{\emo}[1]{\raise1.0ex\hbox{\normalfont#1}}

\title{Conditional Poisson Stochastic Beam Search}

\author{
Clara Meister\emo{\ethz}~\;~Afra Amini\emo{\ethz}~\;~Tim Vieira\emo{\jhu}~\;~Ryan Cotterell\emo{\ucambridge,\ethz} \\
  \emo{\ethz}ETH Z\"{u}rich~\;~\emo{\jhu}Johns Hopkins University~\;~\emo{\ucambridge}University of Cambridge \\
 \texttt{\href{mailto:clara.meister@inf.ethz.ch}{clara.meister@inf.ethz.ch}}~\;~\texttt{\href{mailto:aamini@student.ethz.ch}{aamini@student.ethz.ch}}\\ \texttt{\href{mailto:tim.f.vieira@gmail.com}{tim.f.vieira@gmail.com}}~\;~\texttt{\href{mailto:ryan.cotterell@inf.ethz.ch}{ryan.cotterell@inf.ethz.ch}}
}

\date{}

\begin{document}

\maketitle
\setlength{\abovedisplayskip}{8pt}
\setlength{\belowdisplayskip}{8pt}
\setlength{\abovedisplayshortskip}{6pt}
\setlength{\belowdisplayshortskip}{6pt}
\begin{abstract}
Beam search is the default decoding strategy for many sequence generation tasks in NLP.
The set of approximate $K$-best items returned by the algorithm is a useful summary of the distribution for many applications; however, the candidates typically exhibit high overlap and may give a highly biased estimate for expectations under our model.
These problems can be addressed by instead using stochastic decoding strategies. 
In this work, we propose a new method for turning beam search into a stochastic process: Conditional Poisson stochastic beam search.  Rather than taking the maximizing set at each iteration, we sample $K$ candidates without replacement according to the conditional Poisson sampling design. We view this as a more natural alternative to \citet{kool2019stochastic}'s stochastic beam search (SBS).
Furthermore, we show how samples generated under the CPSBS design can be used to build consistent estimators and sample diverse sets from sequence models. 
In our experiments, we observe CPSBS produces lower variance and more efficient estimators than SBS, even showing improvements in high entropy settings.\footnote{Our codebase is publically available at \url{https://github.com/rycolab/cpsbs}.}\looseness=-1
\end{abstract}

\section{Introduction}
Many NLP tasks require the prediction of structured outputs, such as sentences or parse trees, either during decoding or as part of a training algorithm. 
For today's neural architectures, beam search \cite{reddy-1977} has become the decoding algorithm of choice due to its efficiency and empirical performance \cite{AAAI1714571, edunov-etal-2018-understanding, XLNET,meister-etal-2020-best}.
Beam search is a deterministic method, which invites a natural question: 
What is the proper \emph{stochastic} generalization of beam search?
Several recent papers have investigated this question \cite{kool2019stochastic,Kool2020Estimating,shi2020uniquerandomizer}.
Here we build on this line of work and introduce an alternative stochastic beam search that the authors contend is a more faithful stochasticization of the original algorithm in that it recovers standard beam search as a special case.
We name our algorithm conditional Poisson stochastic beam search (CPSBS) 
as we draw on the conditional Poisson sampling scheme \cite{hajek1964} in its construction.
The relationship between CPSBS and other common decoding strategies is displayed visually in \cref{tab:summmary table}.\looseness=-1

At every iteration, CPSBS replaces the top-$K$ operator in the beam search algorithm with conditional Poisson sampling, resulting in a decoding strategy that generates samples-without-replacement. 
Importantly, annealing our sampling distribution at each time step turns local sampling into a local top-$K$ computation and thereby recovers beam search.
We subsequently show that these samples can be used to construct a statistically consistent estimator for the expected value of an arbitrary function of the output.\looseness=-1

\begin{table}
\setlength\arrayrulewidth{1.5pt}
    \setlength{\extrarowheight}{2pt}
    \adjustbox{max width=\linewidth}{
    \begin{tabular}{|r|c|c|}
      \Cline{2-3}{1.5pt}
      \multicolumn{1}{c|}{\tikz{\node[below left, inner sep=1pt] (def) {\footnotesize Operator};%
      \node[above right,inner sep=1pt] (abc) {\footnotesize Set Size};%
      \draw [line width=1.5pt](def.north west|-abc.north west) -- (def.south east-|abc.south east);}}
       & \cellcolor{yellow1}$K=1$  & \cellcolor{yellow2}$K>1$ \\\Cline{1-3}{1.5pt}
      
      \rowcolor{blue1}\multicolumn{1}{|r|}{$\mathrm{argmax}$} & \cellcolor{green1}Greedy Search & \cellcolor{green2}Beam Search \\\Cline{1-3}{1.5pt}
      
       \rowcolor{blue2}& \cellcolor{green3}Ancestral& \cellcolor{green4}{\bf Conditional} \\[-1pt]
       
       \rowcolor{blue2} \multirow{-2}{*}{$\mathrm{sample}$} &\cellcolor{green3}Sampling &\cellcolor{green4}{\bf Poisson Beams} \\\Cline{1-3}{1.5pt}
    \end{tabular}}
    \caption{A comparison of beam-based decoding algorithms for sequence models, by solution set size and objective. The $\argmax$ and sample variants are related through annealing:
    As the annealing parameter of the distribution $\tau\! \rightarrow \!0$, sampling turns into computing an $\argmax$ (see \cref{prop:anneal}).\looseness=-1
    }\label{tab:summmary table}
    \vspace{-1em}
  \end{table}

In our experiments with neural machine translation models,  we observe that CPSBS leads to better estimates of expected \bleu and conditional model entropy than SBS and the sum-and-sample estimator \cite{Kool2020Estimating}, distinctly outperforming Monte Carlo sampling for both small sample sizes and low temperatures. Furthermore, we find that CPSBS can be used as a diverse sampling strategy. 
We take these results as confirmation that CPSBS is a useful tool in the newfound arsenal of sampling strategies for neural sequence models.

\section{Beam Search}\label{sec:gen}
In this section, we overview the necessary background on neural sequence models
and beam search in order to motivate our algorithm in \cref{sec:cps}.

\paragraph{Neural Sequence Models.}
We consider locally normalized probabilistic models over sequences $\yy$:\looseness=-1
\begin{equation}\label{eq:model}
    \ptheta(\yy) = \prod_{t=1}^{|\yy|}\ptheta(y_t \mid \yy_{<t})
\end{equation}
\noindent where $\yy$ is a member of a set of well-formed outputs $\calY$. In the context of language generation models, well-formed outputs are sequences of tokens $\yy = \langle y_1,y_2, \dots \rangle$ from a vocabulary $\vocab$; all $\yy \in \calY$ begin and end with special tokens \bos and \eos, respectively. We use $\yy_{<t}$ to represent the subsequence $\langle y_1,\dots,y_{t-1}\rangle$. In this work, we consider the setting where the maximum sequence length is upper-bounded; we denote this upper bound $\nmax > 0$.
Without loss of generality, we may condition $\ptheta$
on an input $\xx$, as is necessary for machine translation and other conditional generation tasks.

\paragraph{Beam Search.}
Beam search is a commonly used search heuristic for finding an approximate solution to the following optimization problem:
\begin{equation}
    \yy^\star = \argmax_{\yy \in \calY}\, \log \ptheta(\yy) 
    \label{eq:MAP}
\end{equation}
Its most straightforward interpretation is as a pruned version of
breadth-first search, where the breadth of the search is narrowed to the top-$K$ candidates. 
However, here we will present beam
search in a nonstandard lens \cite{meister-etal-2020-beam,meister-etal-2021-determinantal} in order to emphasize the connection with our stochastic generalization in \cref{sec:cps}.
Specifically, we present the algorithm
as iteratively finding the highest-scoring
set under a specific set function.

Under this paradigm, the initial beam $Y_0$ contains only the \bos token.  At subsequent steps $t=1,\ldots, T$, beam search  selects the $K$ highest-scoring candidates from the set $Y_{t-1} \circ \vocab$ that we define below:\footnote{Sequences already ending in \eos are not extended by $y \in \vocab$ and are simply added to the set ``as is.''} 
\begin{equation}\label{eq:bt}
    Y_{t-1} \circ \vocab \defeq \{ \yy \circ y \mid \yy \in Y_{t-1} \textbf{ and } y \in \vocab \}
\end{equation}
\noindent
where $\circ$ is sequence concatenation. 
Those candidate sets with collectively higher probability under the model $p$ have higher score. 
This process continues until all $\yy \in Y_t$ end in \eos, or $t = \nmax$. For notational ease, we define $\calB_t \defeq  Y_{t-1} \circ \vocab$; throughout this paper, we will assume $|\calB_t| = N$ and identify the elements of $\calB_t = \{\yy_{\leq t}^{(1)}, \ldots, \yy_{\leq t}^{(N)}\}$ with the integers $\{1, \ldots, N\}$. \looseness=-1

We can formulate the time-step dependent set function whose $\argmax$
beam search finds as \begin{equation}
\Q(Y_t \mid Y_{t-1}) \!\defpropto\! \begin{cases}
\prod_{n \in Y_t} w_n & \textbf{if } |Y_t| \!=\! K \\
0 & \textbf{otherwise}
\end{cases}\label{eq:q}
\end{equation}
\noindent where $w_n$ is the weight of the $n^{\text{th}}$ element of $B_t$. 
To recover beam search, we set our weights equal to probabilities under a model $p$, i.e., $w_n = p\!\left(\yy_{\leq t}^{(n)}\right)$. 
Note that we leave the constraint that $Y \subseteq B_t$ implicit in \cref{eq:q}.
As should be clear from notation, this set function only assigns nonzero scores to
subsets of $Y_{t-1} \circ \vocab$ of size exactly $K$ and the assigned score is proportional to
the product of the probability of the candidates under the model $\ptheta$. 
Putting this all together, beam search
may be viewed as the following iterative process:
\begin{tcolorbox}[ams align,colback=gray!35!white,colframe=white!5!white,arc=0pt,outer arc=0pt]%
Y_0 &= \{\bos \} \label{eq:it_subset} \\
Y_t &  = \argmax_{Y_t' \subseteq B_t}\,\,\Q(Y_t' \mid Y_{t-1})\label{eq:beam_subset} \\
& \!\!\!\!\!\textbf{return}\,\, Y_T
\end{tcolorbox}

\section{Conditional Poisson Stochastic Beams}\label{sec:cps}
Our paper capitalizes on a very simple observation: Rather than taking its $\argmax$, we may renormalize \cref{eq:q} into a distribution and sample-without-replacement a size $K$ set at each iteration:

  \begin{tcolorbox}[ams align,colback=gray!35!white,colframe=white!5!white,arc=0pt,outer arc=0pt]
\!\!\!Y_0 &= \{\bos \}  \\
Y_t &  \sim \Q(\cdot \mid Y_{t-1}) \label{eq:cp-sampling} \\
& \!\!\!\!\!\textbf{return}\,\, Y_T
    \end{tcolorbox}

\noindent This recursion corresponds to performing \defn{conditional Poisson sampling} (CPS; \citealt{hajek1964}; see \cref{app:tutorial} for overview), a common sampling-without-replacement design \cite{tille_book},\footnote{A sampling design is a probability distribution over sets of samples.} at every time step. Thus we term this scheme conditional Poisson stochastic beam search.
CPSBS gives us a probability distribution over \emph{sets} of candidates of size $K$, i.e., the final beam $Y_T$. 
We denote the CPSBS distribution $\poisson$ and we write $Y_T \sim \poisson$ to indicate that $Y_T$ is the stochastic beam at the end of a sampling run.
We may write $P(Y_T)$ as a marginal probability, summing over all sequences of beams that could have resulted in $Y_T$:\footnote{This formulation reveals that it is wildly intractable to compute $P(Y_T)$. }
\begin{align}\label{eq:inc-prob}
    \poisson(Y_T) = \sum_{Y_1} \!{\cdots}\! \sum_{Y_{T-1}} \prod_{t=1}^T \Q(Y_t \mid Y_{t-1})
\end{align}
Note the structural zeros of $\Q$ prevent any incompatible sequence of beams. 
We provide a theoretical analysis of the scheme in \cref{sec:estimator} and an empirical analysis in \cref{sec:experiments}.\looseness=-1

\paragraph{Normalizing $\Q(\cdot \mid Y_{t-1})$.}
At each time step $t$, we compute $\Q(\cdot \mid Y_{t-1})$---a distribution over subsets of size $K$ of a base set $\calB_t$---using the CPS design. 
The normalizing constant for this distribution is defined as\looseness=-1
\begin{equation}\label{eq:naive}
    \Z_t \defeq \sum_{\substack{Y_t \subseteq \calB_t, \\ |Y_t| = K}} \prod_{n \in Y_t} w_n \,
\end{equation}
Despite there being exponentially many summands, we can sum over all ${N \choose K}$ subsets in $\bigo{N K}$ time via the following recurrence relation:\footnote{The reader may recognize this recurrence as the weighted generalization of Pascal's triangle, ${n \choose k} = {n-1 \choose k} + {n-1 \choose k-1}$, which is why we chose the notation $\Weight{n}{k}$.}\looseness=-1
\begin{equation}
\Weight{n}{k} = \begin{cases}
1 & \!\!\!\!\textbf{if } k = 0 \\ 
\Weight{n-1}{k} + w_n \Weight{n-1}{k-1} & \!\!\!\!\textbf{if } k \in (0, n) \\
0 & \!\!\!\!\textbf{otherwise }
\end{cases}\nonumber
\end{equation}
We give complete pseudocode in \cref{app:pseudocode}. Correctness of this algorithm is shown in \citet{taskar_dpp}.
The normalizing constant can then be efficiently computed as 
\begin{equation}
    \Z_t = \Weight{N}{K}
\end{equation}
\paragraph{Sampling from $\Q(\cdot \mid Y_{t-1})$.} 
We can efficiently sample sets from $\Q(\cdot \mid Y_{t-1})$ using the algorithm below:
\begin{algorithmic}[1]\label{alg:sample}
\State $Y_t \gets \emptyset$ \rightcomment{Initialization}
\For{$n = N \ldots 1$} 
\State $k \gets K - |Y_t|$ \rightcomment{Number of remaining elements}
\State Add the $n^{\text{th}}$ element of $B_t$ to $Y_t$ with prob.
$$
\frac{w_n \, \Weight{n-1}{k-1}}{\Weight{n}{k}}
$$
\EndFor
\State \Return $Y_t$ \Comment{Guaranteed to have size $K$}
\end{algorithmic}
\noindent In words, the algorithm considers adding each element one at a time until $K$ elements have been sampled.
Notice that line 4 adjusts the probability of sampling item $n$ given that $|Y_t|$ items have already been sampled, which ensures that exactly $K$ elements are sampled at termination.

\paragraph{Setting $w_n$.}\label{prop:anneal} 
The weight assigned to the $n^{\text{th}}$ item of $\calB_t$ directly affects its probability of being included in the sampled set, i.e., $\mathrm{Pr}\left(\yy_{\leq t}^{(n)}\in Y_t\right)$, also termed an item's \defn{inclusion probability}.
In this paper, we write $\piQ\Big(\yy_{\leq t}^{(n)}\!\mid\!Y_{t-1}\Big)$ to denote the inclusion probability under the distribution $\Q(\cdot \mid Y_{t-1})$, defined as:\looseness=-1
\begin{align}
    \piQ\Big(\yy_{\leq t}^{(n)}&\mid Y_{t-1}\Big)  \\
    &\defeq \sum_{Y_t} \Q(Y_t\mid Y_{t-1}) \mathbbm{1}\{\yy_{\leq t}^{(n)} \in Y_t \} \nonumber 
\end{align}
One strategy is to choose $w_n$ at time step $t$ such that $ \piQ\Big(\yy_{\leq t}^{(n)}\mid Y_{t-1}\Big) \approx p(\yy_{\leq t}^{(n)})$.
This choice recovers beam search when we \emph{anneal} our chosen weights $w_n \mapsto w_n^{1/\tau}$: as the temperature parameter $\tau \to 0$, the CP distribution will assign probability 1 to the set containing the top-$K$ elements.\footnote{In the event of ties, annealed CP will converge to a distribution that breaks ties uniformly at random.}
 
Finding $w_n$'s that result in pre-specified inclusion probabilities is possible, but it requires solving a numerical optimization problem \cite{aires1999algorithms,GRAFSTROM20092111}. 
Further, in CPSBS, we will be sampling from a different distribution at each time step and it would be quite slow to solve the numerical optimization problem each iteration. 
Luckily, the choice of $w_n = p(\yy_{\leq t}^{(n)})/(1-p(\yy_{\leq t}^{(n)}))$
yields a good approximation to the target inclusion probabilities in both theory and practice \cite{hajek1981sampling,bondesson,aires1999algorithms}.%

\section{Statistical Estimation with Conditional Poisson Stochastic Beam Search}\label{sec:estimator}
In this section, we discuss statistical estimation with CPSBS samples.  
To that end, we construct two estimators with different properties.
However, only the second estimator provides good performance in practice, which is discussed later in \cref{sec:experiments}.
\subsection{The Horvitz--Thompson Estimator}\label{sec:ht}
We build upon the Horvitz--Thompson (HT) estimator \cite{HorvitzThompson1952}, 
which is a common technique for estimation from sampling-without-replacement (SWOR) schemes.

Let $\func: \calY \rightarrow \mathbb{R}^d$ be a function whose expected value under $p$ we seek to approximate: 
\begin{equation}\label{eq:expectation}
    \mathbb{E}_{\yy \sim p}\left[ \func(\yy) \right]= \sum_{\yy \in \calY} p(\yy) \func(\yy)
\end{equation}
The Monte Carlo estimator of the above quantity is\looseness=-1
\begin{equation}\label{eq:montecarlo}
    G_{\mathrm{MC}} \defeq \frac{1}{M}\sum_{m=1}^M f(\yy^{(m)})
\end{equation}where $\yy^{(m)} \overset{\mathrm{i.i.d.}}{\sim} p$.
 However, in the special case of sampling from a finite population---which is extremely common in NLP---it can be very wasteful. For example, if a distribution is very peaked, it will sample the same item repeatedly; this could lead to inaccurate approximations for some $f$.
 As a consequence, the mean square error (MSE) of the estimator with respect to $\mathbb{E}_{\yy \sim p}\left[ \func(\yy) \right]$ can be quite high for small $M$. Indeed, we see this empirically in \cref{fig:rmse-experiment}.\looseness=-1

Taking samples without replacement allows us to cover more of the support of $p$ in our estimate of $\mathbb{E}_{\yy \sim p}\left[ \func(\yy) \right]$. 
However, we must take into account that our samples are no longer independent, i.e., $\yy^{(m)} \overset{\mathrm{i.i.d.}}{\not\sim} p$.
We now define the HT estimator, using notation specifically for the case of CPSBS:
\begin{align}\label{eq:ht}
    G_{\mathrm{HT}} \defeq \sum_{\yy \in Y_T} \frac{p(\yy)}{\pipoisson(\yy)} \func(\yy) 
\end{align}
As should be clear from notation, we assume $Y_T \sim \poisson$; further, we use $\pipoisson(\yy)$ to denote the \defn{inclusion probability} of $\yy$ under CPSBS, i.e., the probability of sampling a set $Y_T \sim \poisson$ that contains the element $\yy$:\looseness=-1
\begin{align}
    &\pipoisson(\yy) 
    = \sum_{Y_T}\! \poisson( Y_T) \mathbbm{1}\{\yy \in Y_T \} \label{eq:monster-inclusion}\\
    &= \sum_{Y_1} \!{\cdots}\! \sum_{Y_{T}} \prod_{t=1}^T \Q(Y_t \mid Y_{t-1})\mathbbm{1}\left\{\yy_{\leq t} \in Y_t \right\} \nonumber
\end{align}
\noindent In \cref{eq:ht}, the distribution $\pipoisson$ may be viewed as a  proposal distribution in the sense of importance sampling \cite{mcbook} and $1/\pipoisson(\yy)$ as the corresponding importance weight corrections. If we can exactly compute $\pipoisson$, then
the HT estimator
is unbiased\footnote{Note that it is common to normalize \cref{eq:ht} by the sum of importance weights, i.e., divide $G_{\mathrm{HT}}$ by the sum $\sum_{\yy \in Y{T}}\piQ(\yy)$. While this leads to a biased estimator, it can significantly reduce variance, which is often worthwhile.\looseness=-1} (see \cref{app:ht} for proof). However, the summation in \cref{eq:monster-inclusion} is  intractable so we resort to estimation.\looseness=-1

\subsection{Estimating Inclusion Probabilities}
In this section, we develop statistical estimators of the inclusion probabilities under conditional Poisson stochastic beam search. Note that in order to maintain the unbiasedness of the HT estimator, we must estimate the reciprocal inclusion probabilities.\footnote{Since by Jensen's inequality $\expect{}{1/X} \geq 1/\expect{}{X}$ for $X\in \mathbb{R}_+$, the reciprocal of an unbiased estimate of $\pipoisson(\yy)$ is not an unbiased estimate of $1/\pipoisson(\yy)$\looseness=-1} However, these are not straightforward to estimate. 
Thus, we attempt to estimate the inclusion probabilities directly and take the reciprocal of this estimator.
This strategy leads to a consistent, but biased, estimator.
An important caveat: the analysis in this section only applies to the estimators of the inclusion probabilities themselves. 
Further analysis may be undertaken to analyze the variance of the HT estimators that make use of these estimators.

\subsubsection{Na{\"i}ve Monte Carlo }
One obvious way to derive an inclusion probability estimator 
is via Monte Carlo estimation:
\begin{equation}\label{eq:mc-naive}
    \pipoissonmc(\yy) \defeq \frac{1}{M}\sum_{m=1}^M \mathbbm{1}\Big\{\yy \in Y_{T}^{(m)}\Big\}
\end{equation}
where $Y^{(m)} \sim \poisson$.%
\begin{restatable}{prop}{proptwo}\label{prop:estimator-2}
Eq.~\ref{eq:mc-naive} has the following two properties:
\begin{enumerate}[label=\roman*)]
\item $\pipoissonmc$ is an unbiased estimator of $\pipoisson$ and
\begin{equation}\label{eq:mc-var}
\mathbb{V}\left[\pipoissonmc\right] = \frac{1}{M} \left(\pipoisson(\yy ) - \pipoisson(\yy)^2\right)
\end{equation}
 \item ${1}/{\pipoissonmc}$ is a consistent estimator of ${1}/{\pipoisson}$ with asymptotic variance \begin{equation}
\asvar\left[\frac{1}{\pipoissonmc(\yy)}\right]\!=\! \frac{1}{M}\left(\frac{1}{\pipoisson(\yy)^3}\!-\! \frac{1}{\pipoisson(\yy)^2}\right)
\end{equation}
\end{enumerate}
Here $\asvar$ denotes the \defn{asymptotic variance}, which is the variance after the number of samples $M$ is large enough such that the central limit theorem has kicked in \cite{bickel2015mathematical}.
\end{restatable}
\begin{proof}
Proof given in \cref{app:proptwo}. 
\end{proof}
\noindent Qualitatively, what this result tells us is that if we are asking about the inverse inclusion probability of a candidate with a low inclusion probability, our estimator may have very high variance. 
Indeed, it is unlikely that we could derive an estimator without this qualitative property due to the presence of the inverse. 
Moreover, the estimator given in \cref{eq:mc-naive} is not of practical use: If we are interested in the inverse inclusion probability of a specific candidate $\yy$, then we may have to sample a very large number of beams until we eventually sample one that actually contains $\yy$.
In practice, what this means is that our estimate of the inclusion probably for a rare $\yy$ will often be zero, which we cannot invert.\footnote{One solution would be to smooth our estimates of the inclusion probabilities, adding a small $\varepsilon$ to ensure that we do not divide by zero, but the authors find our next approach to be more methodologically sound.} Instead, we pursue an importance sampling strategy for estimating $\pipoisson(\yy)$, which we outline in the next section.\looseness=-1

\subsubsection{Importance Sampling}\label{sec:low-var}
We now turn to an inclusion probability estimator that is based on importance sampling.
Recall from \cref{eq:monster-inclusion} that the inclusion probability for $\yy$ is a massive summation over sequences of possible beams ${Y_1}, \ldots, Y_{T}$ that could have generated $\yy$.
Rather than computing the sum, we will estimate the sum through taking samples.
Our procedure starts by generating \defn{hindsight samples} $\hY{1}, \ldots, \hY{T}$ from the following proposal distribution that is conditioned on $\yy$:\looseness=-1
\begin{equation}\label{eq:beam_cp_p}
\tildeQ(\hY{t}  \mid \hY{t-1}, \yy) 
\defeq \frac{\Q(\hY{t}  \mid \hY{t-1})}{\piQ(\yy_{\leq t}\mid \hY{t-1})}
\end{equation}
In words, $\tildeQ$ is $\Q$ conditioned on its sets $Y_t$ containing the prefix $\yt$ (thus it is always the case that $\yt\in \hY{t}$).\footnote{This proposal distribution can be realized through a minor modification of our algorithm in \cref{alg:sample}, where $w(\yy)$ corresponding to $\yti$ is placed at the beginning and added to $Y_t$ deterministically.\looseness=-1}
For brevity, we omit an explicit notational dependence of $\hY{t}$ and $\tildeQ$ on $\yy$.
\begin{restatable}{lemma}{lemmaone}\label{lem:decomposition}
The joint proposal distribution $\tildepoisson(\hY{1}, \ldots, \hY{\nmax}) \defeq \prod_{t=1}^T \tildeQ(\hY{t} \mid \hY{t-1})$ may be expressed in terms of $\poisson$ as follows:
\begin{align}\label{eq:tildepoisson}
    \tildepoisson&(\hY{1}, \ldots, \hY{\nmax}) = \frac{\poisson(\hY{1}, \ldots, \hY{\nmax})}{\prod_{t=1}^{T} \piQ(\yy_{\leq t}\mid \hY{t-1})} 
\end{align}
where we define $\poisson(\hY{1}, \ldots, \hY{\nmax}) \defeq \prod_{t=1}^T \Q(\hY{t} \mid \hY{t-1})$ as the joint probability of the beams $\hY{1}, \ldots, \hY{\nmax}$ under the original distributions $\Q$. 
We omit that both $\poisson$ and $\tildepoisson$ are conditioned on $Y_0$. 
\end{restatable}
\begin{proof}
See \cref{app:lemmaone}.
\end{proof}
\noindent %
In terms of computation, \cref{eq:beam_cp_p} makes use of the fact that the per-time-step inclusion probability $\piQ(\yt)$
for a given $\Q$ can be computed efficiently with dynamic programming using the following identity:\looseness=-1
\begin{align}
    \piQ(\yy_{\leq t}^{(n)}\mid Y_{t-1}) &\defeq \sum_Y \Q(Y_t)\mathbbm{1}\Big\{\yy_{\leq t}^{(n)} \in Y_t\Big\} \nonumber \\
&= \frac{w_n}{\Z} \frac{\partial \Z}{\partial w_n} \label{eq:incl-from-grad}
\end{align}
For completeness, we give pseudocode in \cref{app:pseudocode}.
Given samples $\hY{T}^{(m)} \sim \tildepoisson$ for $\tildepoisson$ defined in \cref{eq:tildepoisson} with respect to a given $\yy$, we propose the following unbiased estimator of inclusion probabilities:\looseness=-1
\begin{equation}\label{eq:is-good}
   \pipoissonis(\yy) \defeq \frac{1}{M}\sum_{m=1}^M \prod_{t=1}^T \piQ(\yt\mid \hY{t-1}^{(m)})
\end{equation}
where $\yt$ is a prefix of $\yy$.
One simple derivation of \cref{eq:is-good} is as an importance sampler.
We start with the equality given in \cref{eq:monster-inclusion} and perform the standard algebraic manipulations witnessed in importance sampling:
\allowdisplaybreaks
\begin{align}
    &\sum_{Y_T} \poisson(Y_T) \mathbbm{1}\{\yy \in Y_T \} \\
   &=  \sum_{Y_1} \cdots \sum_{Y_T}\poisson(Y_1, \ldots, Y_T) \mathbbm{1}\{\yy \in Y_T \} \nonumber \\
   &= \sum_{\hY{1}}\cdots \sum_{\hY{T}} \poisson(\hY{1}, \ldots, \hY{T}) \frac{\tildepoisson(\hY{1}, \ldots, \hY{T})}{\tildepoisson(\hY{1}, \ldots, \hY{T})}  \nonumber\\
   &=  \sum_{\hY{1}}\cdots \sum_{\hY{T}} \tildepoisson(\hY{1}, \ldots, \hY{T})  \frac{\poisson(\hY{1}, \ldots, \hY{T})}{\tildepoisson(\hY{1}, \ldots, \hY{T})} \nonumber \\
     &\overset{\textit(i)} {=} \sum_{\hY{1}}\cdots \sum_{\hY{T}} \tildepoisson(\hY{1}, \ldots, \hY{T}) \prod_{t=1}^T \piQ(\yy_{\leq t}\mid \hY{t-1})  \nonumber %
\end{align}
where equality \textit{(i)} above follows from \cref{lem:decomposition}.
This derivation serves as a simple proof that \cref{eq:is-good} inherits unbiasedness from \cref{eq:ht}.

\begin{restatable}{prop}{propthree}\label{prop:is}
\cref{eq:is-good} has the following two properties:
\begin{enumerate}[label=\roman*)]
\item $\pipoissonis$ is an unbiased estimator of $\pipoisson$;
\item The estimator of the inverse inclusion probabilities $1/\pipoissonis(\yy)$
is consistent with the following upper bound on the asymptotic variance:

\begin{align}
\asvar&\left[\frac{1}{\pipoissonis(\yy)}\right] \leq \frac{1}{M} \frac{r - 1}{\pipoisson(\yy)^2}
\end{align}
where we assume that the following bound:
\begin{equation}
\frac{\prod_{t=1}^T \piQ(\yy_{\leq t}\mid \hY{t-1})}{\pipoisson(\yy)} \leq r
\end{equation}
holds for all $\hY{1}, \ldots, \hY{T}$.
\end{enumerate}
\end{restatable}
\begin{proof}
Proof given in \cref{app:propthree}.
\end{proof}
\noindent \cref{prop:is} tells us that we can use \cref{eq:is-good} to construct a consistent estimator of the inverse inclusion probabilities. 
Moreover, assuming $\mathrm{Pr}\left(\yy \!\in\! Y_T \right) > 0$, then we have that the importance sampling yields an estimate  $\pipoissonis(\yy) > 0$, unlike the Monte Carlo estimator $\pipoissonmc(\yy)$.
We further see that, to the extent that $\prod_{t=1}^T \piQ(\yy_{\leq t}\mid \hY{t-1})$ approximates $\pipoisson(\yy)$, then we may expect the variance of \cref{eq:is-good} to be small---specifically in comparison to the na{\"i}ve Monte Carlo estimator in \cref{eq:mc-naive}---which is often the case for estimators built using importance sampling techniques when a proposal distribution is chosen judiciously \cite{monte-carlo}.
Thus, given our estimator in \cref{eq:is-good}, we can now construct a practically useful estimator for $\E_{\yy \sim p}\left[  f(\yy)\right]$ using the HT estimator in \cref{eq:ht}. 
In the next section, we observe that this estimator is quite efficient in the sequence model setting.\looseness=-1

\section{Experiments}\label{sec:experiments}

We repeat the analyses performed by \citet{kool2019stochastic}, running experiments on neural machine translation (NMT) models; for reproducibility, we use the pretrained Transformer model for WMT'14 \cite{bojar2014findings} English--French made available by \texttt{fairseq}\footnote{\url{https://github.com/pytorch/fairseq/tree/master/examples/translation}} \cite{ott2019fairseq}. We evaluate on the En-Fr \texttt{newstest2014} set, containing 3003 sentences. Further details can be found in \cref{app:results}. 
Our implementation of CPSBS modifies the beam search algorithm from the \texttt{fairseq} library. We additionally consider the beam search, stochastic beam search, diverse beam search, and ancestral sampling algorithms available in \texttt{fairseq}.\looseness=-1

\subsection{Statistical Estimators for Language Generation Models}
Estimators have a large number of applications in machine learning. For example, the REINFORCE algorithm \cite{reinforce} constructs an estimator for the value of the score function; minimum-Bayes risk decoding \cite{kumar-byrne-2004-minimum} uses an estimate of risk in its optimization problem. 
In this section, we compare estimators for sentence-level \bleu score and conditional model entropy for NMT models. 
Notably, NMT models that are trained to minimize cross-entropy with the empirical distribution\footnote{Label-smoothing \cite{label_smoothing} is typically also employed, which leads to even higher entropy distributions.} are not peaky distributions \cite{ott2018analyzing,eikema2020map}; thus, standard estimation techniques, e.g., Monte Carlo, should generally provide good results. However, we can vary the annealing parameter of our model in order to observe the behavior of our estimator with both high- and low-entropy distributions, making this a comprehensive case study. 
Here the annealed model distribution is computed as\looseness=-1
\begin{equation}
    p_\tau(y_t \mid  \yy_{<t}) \propto p\left(y_t \mid  \yy_{<t}\right)^{\frac{1}{\tau}}
\end{equation}
where we should expect a standard Monte Carlo estimator to provide good results at $\tau$ close to 1 when $p$ is naturally high entropy. We test our estimator in this setting so as to give a comparison in a competitive setting.
Specifically, we assess the performance of our estimator of $\mathbb{E}_{\yy \sim p(\yy \mid \xx)}[f(\yy)]$ given in \cref{eq:ht}---using inclusion probability estimates from \cref{eq:is-good} with $M=1$ and with importance weight normalization---in comparison to three other approaches: Monte Carlo (MC) sampling, the sum-and-sample (SAS) estimator, and stochastic beam search (SBS).\looseness=-1

\paragraph{Monte Carlo.} Under the Monte Carlo sampling scheme with sample size $K$, we estimate the expected value of $f$ under our model using \cref{eq:montecarlo} with a sample $\yy^{(1)},\ldots, \yy^{(K)} \overset{\mathrm{i.i.d.}}{\sim} p$.

\begin{figure*}
    \centering
    \includegraphics[width=\textwidth]{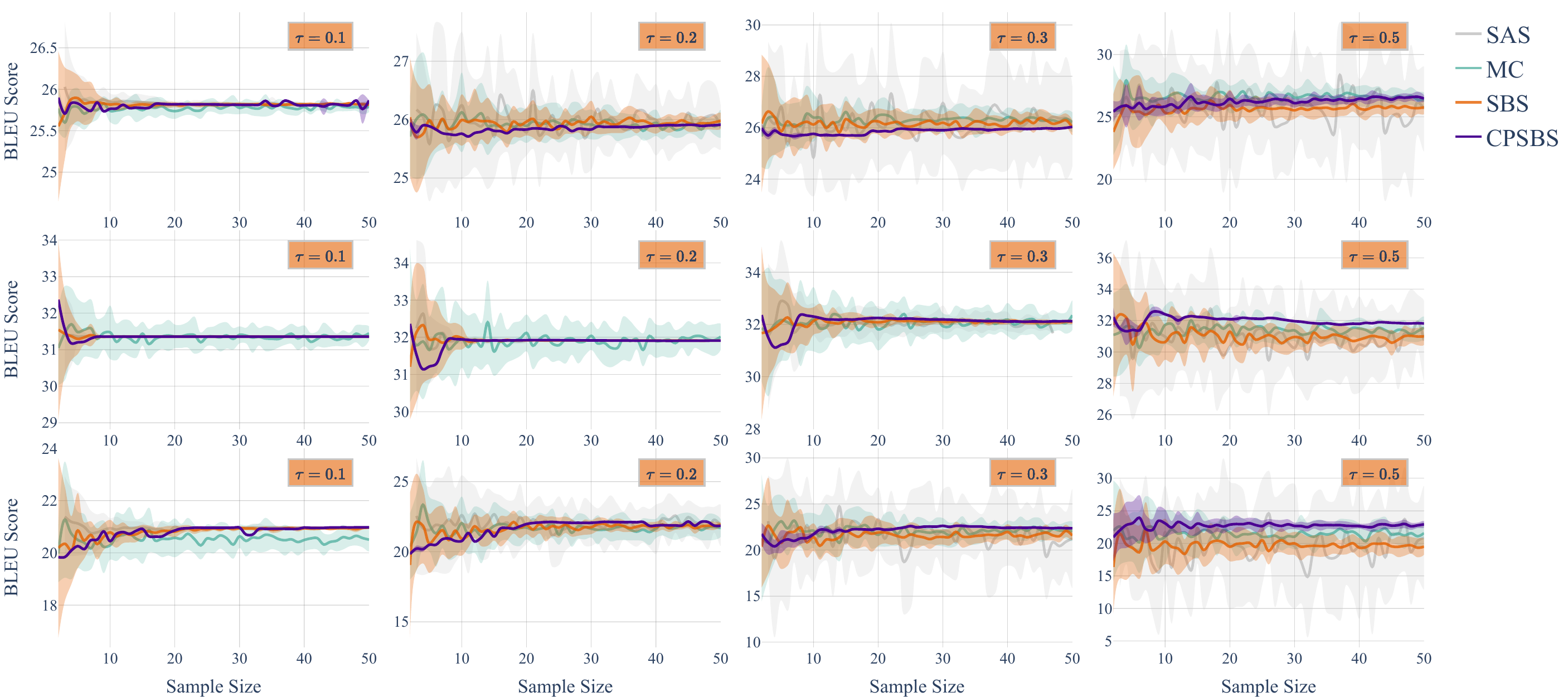}
     \caption{\bleu score estimates for three different sentences using estimators for respective decoding methods. $\tau$ indicates scaling temperature; $\tau$ values and sentences are chosen to mimic \cite{kool2019stochastic}.}
        \label{fig:bleu-experiment}
\end{figure*}
 \paragraph{Sum and Sample.} The sum-and-sample estimator \cite{botev17a,liu2019rao,Kool2020Estimating} is an unbiased estimator that takes as input a deterministically chosen set $Y$ of size $K-1$
    and samples an additional $\yy'$ from the remaining elements, $\mathrm{supp}(p)\setminus Y$, where we obtain the set $Y$ using beam search in our experiments. Formally, the SAS estimator can be written as:
\begin{align}
        G_{\textsc{sas}} \defeq& \sum_{k=1}^{K-1}\ptheta(\yy^{(k)})f(\yy^{(k)})  \\
        & \quad\quad + \left(1-\sum_{k=1}^{K-1}\ptheta(\yy^{(k)}) \right) f(\yy') \nonumber
\end{align}

\paragraph{Stochastic Beam Search.} Stochastic beam search \cite{kool2019stochastic,Kool2020Estimating} is a SWOR algorithm likewise built on top of beam search. The algorithm makes use of truncated Gumbel random variables at each iteration, resulting in a sampling design equivalent to performing the Gumbel-top-$k$ trick \cite{vieira2014gumbel} on the distribution $\ptheta$. Estimators built using SBS likewise follow the Horvitz--Thompson scheme of \cref{eq:ht}; we refer the reader to the original work for inclusion probability computations. They suggest normalizing the estimator by the sum of sample inclusion probabilities to help reduce variance; we therefore likewise perform this normalization in our experiments.\looseness=-1
 \begin{figure*}
    \centering
    \begin{subfigure}[a]{0.9\textwidth}
    \includegraphics[width=\textwidth]{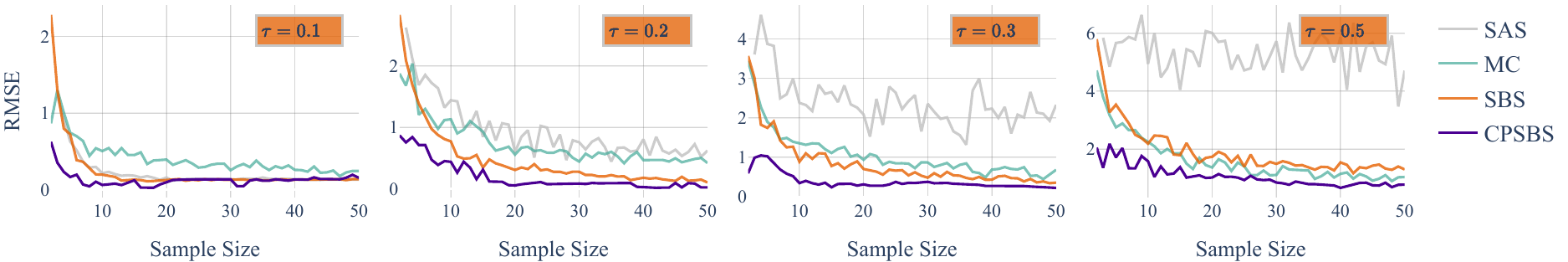}
     \caption{RMSE of \bleu score estimator for different temperatures. Results are averaged across several sentences.\looseness=-1}
        \label{fig:rmse-bleu-experiment}
    \end{subfigure}
    \begin{subfigure}[b]{0.9\textwidth}
    \includegraphics[width=\textwidth]{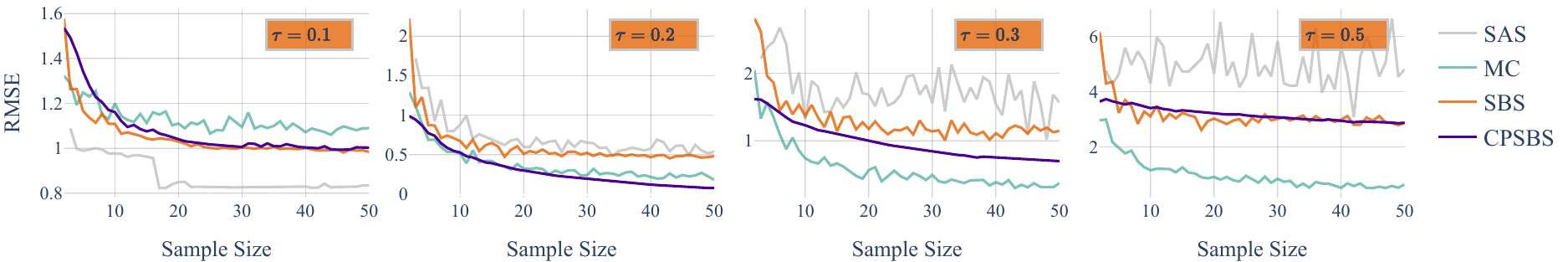}
     \caption{RMSE of conditional model entropy estimator for various temperatures. Results are averaged across several sentences. We see a larger bias under both CPSBS and SBS at higher temperatures in these experiments.\looseness=-1}
        \label{fig:rmse-entropy-experiment}
    \end{subfigure}
    \label{fig:rmse-experiment}
    \caption{RMSE evaluations.}
\end{figure*}

To assess the error of our estimator, we compute its root MSE (RMSE) with respect to a baseline result. While computing the exact value of an expectation is typically infeasible in the sequence model setting, we can average our (unbiased) MC estimator in \cref{eq:montecarlo} over multiple runs to create a good baseline. Specifically, we compute our MC estimator 50 times for a large sample size ($K=200$); variances are reported in \cref{app:results}.\looseness=-1
\footnotetext{We refer the reader to the original work \cite{kool2019stochastic} for equations for inclusion probability estimates.}

Probabilistic models for language generation typically have large vocabularies. In this setting, the computation of \cref{eq:naive} is inefficient due to the large number of items in the set that are assigned very small probability under the model. We experiment with truncating this distribution such that the set of possible extensions of a sequence consist only of the highest probability tokens within the core $n$\% of probability mass ($0.99$ in our experiments), similar to the process in nucleus sampling \cite{holtzman2019curious}. We compare this approach to the original algorithm design in \cref{app:results} and see that empirically, results are virtually unchanged; the following results use this method. We also compare the decoding time of different sampling methods in \cref{fig:runtime}.

\paragraph{\bleu Score Estimation.}
\bleu \cite{Papineni:2002:BMA:1073083.1073135} is a widely used automatic evaluation metric for the quality of machine-generated translations. 
Estimates of \bleu score are used in minimum risk training \cite{shen-etal-2016-minimum} and reinforcement learning-based approaches \cite{RanzatoCAZ15} to machine translation.
As such, accurate and low-variance estimates are critical for the algorithms' performance.
 Formally, we estimate the expected value of $f(\yy) = \bleu(\xx, \yy)$, whose dependence on $\xx$ we leave implicit, under our NMT model $p$  for reference translation $\xx$. For comparison, we use the same sentences and similar annealing temperatures\footnote{Results for $\tau=0.05$ converged rapidly for all estimators, thus not providing an interesting comparison. } $\tau$ evaluated by \citet{kool2019stochastic}. 
We repeat the sampling 20 times and plot the value and standard deviation (indicated by shaded region) 
of different estimators in \cref{fig:bleu-experiment}.
From \cref{fig:bleu-experiment}, we can see that CPSBS has lower variance than our baseline estimators across all temperatures and data points.\footnote{The sampling distribution at $n=1$ is not the same across strategies, hence the difference in variances even at $n=1$.\looseness=-1} 
Especially in the low temperature setting, our estimator converges rapidly with minor deviation from the exact values even for small sample sizes. Additionally, in \cref{fig:rmse-bleu-experiment} we see that the RMSE is typically quite low except at higher temperatures. 
In such cases, we observe the effects of biasedness, similar to \citet{kool2019stochastic}'s observations.\looseness=-1

\paragraph{Conditional Entropy Estimation.}
We perform similar experiments for estimates of a model's conditional entropy, i.e., $f(\yy) = - \log  \ptheta(\yy \mid \xx)$, whose dependence on $\xx$ we again leave implicit.
We show results in \cref{fig:rmse-entropy-experiment}, with plots of the value in \cref{app:results} since results are quite similar to \cref{fig:bleu-experiment}. We see further confirmation that our estimator built on CPSBS is generally quite efficient.

\subsection{Diverse Sampling}
We show how CPSBS can be used as a diverse set sampling design for language generation models. We generate a sample of translations $Y_T \sim \poisson$, i.e., according to the CPSBS scheme, where weights are set as $w_n = p(\yti)/(1-p(\yti))$ at each time step, as recommended in \cref{sec:cps}. 
In \cref{fig:bleu-diversity}, we show the trade-off between minimum, average and maximum sentence-level \bleu score (as a quality measure) and $n$-gram diversity, where we define $n$-gram diversity $D$ as the average fraction of unique vs. total $n$-grams for $n=1,2,3,4$ in a sentence:\looseness=-1
\begin{equation}\label{eq:diversity}
D = \sum_{n=1}^4\frac{\# \textrm{unique $n$-grams in $K$ strings}}{\# \textrm{ $n$-grams in $K$ strings}}
\end{equation}

\begin{figure}
    \centering
    \includegraphics[width=\columnwidth]{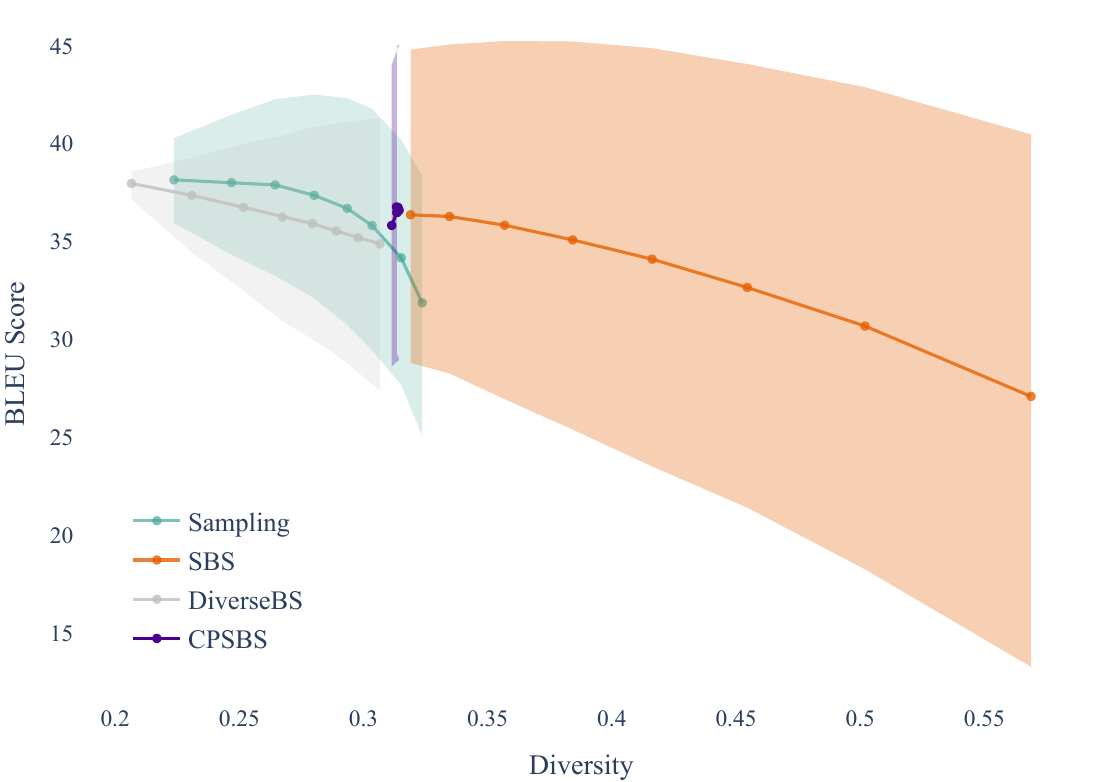}
     \caption{Average ($\pm$ min and max) \bleu score versus diversity for sample size $k = 5$. Points correspond to different annealing temperatures \{0.1, \dots, 0.8\}. Results for $k=10,20$ show very similar trends. }
    \label{fig:bleu-diversity}
\end{figure}
\noindent Metrics are averaged across the corpus. 
We follow the experimental setup of \citet{kool2019stochastic}, using the \texttt{newstest2014} dataset and comparing three different decoding methods: SBS, diverse beam search \citep[DiverseBS;][]{vijayakumar2018diverse} and ancestral sampling.
As in their experiments, we vary the annealing temperature in the range $\{0.1, 0.2, \ldots, 0.8\}$ as a means of encouraging diversity; for DiverseBS we instead vary the \textit{strength parameter} in the same range. 
Interestingly, we see that temperature has virtually no effect on the diversity of the set of results returned by CPSBS. Despite this artifact, for which the authors have not found a theoretical justification,\footnote{While scaling sampling weights by a constant should not change the distribution $\poisson$, an $\exp$ transformation of weights---which is the computation performed by temperature annealing---should.} the set returned by CPSBS is still overall more diverse (position on $x$-axis) than results returned by DiverseBS and reflect better min, max, and average \bleu in comparison to random sampling. SBS provides a better spectrum for the diversity and \bleu tradeoff; we thus recommend SBS when diverse sets are desired.\looseness=-1

\section{Conclusion}

In this work, we present conditional Poisson stochastic beam search, a sampling-without-replacement strategy for sequence models. Through a simple modification to beam search, we turn this mainstay decoding algorithm into a stochastic process. We derive a low-variance, consistent estimator of inclusion probabilities under this scheme; we then present a general framework for using CPSBS to construct statistical estimators for expectations under sequence models.  In our experiments, we observe a reduction in mean square error, and an increase in sample efficiency, when using our estimator in comparison to several baselines, showing the benefits of CPSBS.
\section*{Acknowledgements}

We thank Darcey Riley, as well as our anonymous reviewers, for their helpful feedback. 

\bibliography{anthology,acl}
\bibliographystyle{acl_natbib}
\clearpage
\newpage
\appendix
\onecolumn 
\section{Conditional Poisson Sampling}\label{app:tutorial}
Here we provide a brief overview of the sampling design at the core of CPSBS: conditional Poisson sampling. We consider a base set $B$ where $|\calB| = N$ and we map the elements of $\calB = \{\yy^{(1)}, \ldots, \yy^{(N)}\}$ to the integers $\{1, \ldots, N\}$.
As a warm up, we first consider (unconditional) Poisson sampling, also known as a Bernoulli point process.  To sample a subset $Y \subseteq \calB$, we do as follows: for each element $\yy \in \calB$, we flip a coin where the odds of heads is $w(\yy)$. 
Then, we simply take $Y$ to be the subset of elements whose coin flips were heads. However, this sampling scheme clearly does not guarantee a sample of $K$ items, which can cause problems in our application; sampling more than $K$ items would make the stochastic beam search process inefficient while sampling fewer than $K$---or even 0---items may not leave us with a large enough set at the end of our iterative process.

If instead, we \emph{condition} on the sets always having a prescribed size $K$, i.e., reject samples where $|Y| \ne K$, we arrive at the \emph{conditional} Poisson process.  Formally, the conditional Poisson distribution is defined over $Y \subseteq \calB$ as follows,
\begin{equation}
Q(Y) \defpropto \begin{cases} 
\prod_{\yy \in Y} w(\yy) & \textbf{if } |Y| \!=\! K \\
0 & \textbf{otherwise}
\end{cases}\label{eq:cp-distrib}
\end{equation}
By analyzing \cref{eq:cp-distrib}, we can see that sets with the largest product of weights are the most likely to be sampled; further, this distribution is invariant to rescaling of weights due to the size requirement. This is similar to the conditions under which beam search chooses the set of $K$ largest weight, i.e., highest scoring, elements. Indeed, we note the extreme similarity between \cref{eq:q} and \cref{eq:cp-distrib}, the only difference being a dependence on a prior set. However, unlike beam search, sets with a lower weight product now have the possibility of being chosen.

\section{Proofs}

\subsection{Unbiasedness of the Horvitz--Thompson Estimator}
\begin{prop}

Given a SWOR design $Q$ over the set $\calB = \{1, \dots, N\}$ with inclusion probabilities $\inc(n)$, the Horvitz--Thompson estimator (\cref{eq:ht}) gives us an unbiased estimator of $\E_{n \sim p} f(n)$, where $\func : B \rightarrow \mathbb{R}^d$ is a function whose expectation under $p$ we seek to approximate.
\end{prop}
\begin{proof}\label{app:ht}
\begin{subequations}
\begin{align}
    \expect{Y \sim Q}{G_{\mathrm{HT}}} &=
    \E_{Y \sim Q} \sum_{n=1}^N \frac{p(n )}{\inc(n )}\, f(n) \\
    &= \E_{Y \sim Q} \sum_{n \in \calB} \frac{p(n )}{\inc(n  )}\,\mathbbm{1}\{n \in Y \}\, f(n) \\
     &= \sum_{n \in \calB} \frac{p(n )}{\inc(n )}\, f(n)\,\E_{Y\sim Q}  \mathbbm{1}\{n \in Y\} \\
          &= \sum_{n \in \calB} \frac{p(n )}{\inc(n  )}\, f(n)\,\inc(n  ) \\
     &= \sum_{n \in \calB} p(n )\, f(n) \\
          &=  \E_{n \sim p}  f(n) 
\end{align}
\end{subequations}
\end{proof}
\subsection{Proofs of Expected Values and Variances of Inclusion Probability Estimators}
\lemmaone*
\begin{proof}\label{app:lemmaone}

Consider the probability of sampling $\hY{1}, \ldots, \hY{\nmax}$ according to $\tildepoisson$.
Algebraic manipulation reveals:
\allowdisplaybreaks
\begin{subequations}\begin{align}
    \tildepoisson(\hY{1}, \ldots, \hY{\nmax})
    &=\frac{\Q(\hY{1}\mid Y_0)}{\piQ(\yy_{\leq 1}\mid Y_0)}\cdots \frac{\Q(\hY{\nmax}\mid \hY{\nmax-1})}{\piQ(\yy_{\leq T}\mid \hY{\nmax-1})} \\
    = &\frac{\poisson(\hY{1}, \ldots, \hY{\nmax})}{\prod_{t=1}^{T} \piQ(\yy_{\leq t} \mid \hY{t-1})}\,\label{eq:ptilddef}
\end{align}\end{subequations}
which proves the identity.
\end{proof}

\proptwo*

\begin{proof}\label{app:proptwo}
\noindent i) The estimator is easily shown to be unbiased:
\begin{equation}
    \E \pipoissonmc(\yy) \defeq \frac{1}{M}\sum_{m=1}^M \mathbbm{1}\Big\{\yy \in Y_T^{(m)}\Big\} = \pipoisson(\yy)
\end{equation}
and its variance may be derived as follows:
\begin{subequations}
\begin{align}
    \mathbb{V}\left[ \pipoissonmc(\yy) \right] &\defeq \mathbb{V}\left[\frac{1}{M}\sum_{m=1}^M \mathbbm{1}\Big\{\yy \in Y_T^{(m)}\Big\} \right] \\
    &=\frac{1}{M} \mathbb{V}\left[ \mathbbm{1}\Big\{\yy \in Y_T^{(m)}\Big\} \right] \\
     &=\frac{1}{M} \left( \E\left(\mathbbm{1}\Big\{\yy \in Y_T^{(m)}\Big\}^2\right) - \E\left(\mathbbm{1}\Big\{\yy \in Y_T^{(m)}\Big\}\right)^2 \right) \\
     &=\frac{1}{M} \left(\pipoisson(\yy) - \pipoisson(\yy)^2 \right)
\end{align}
\end{subequations}
\noindent ii) By the strong law of large numbers, we have
\begin{equation}
    \lim_{M\rightarrow \infty} \frac{1}{M}\sum_{m=1}^M \mathbbm{1}\Big\{\yy \in Y_T^{(m)}\Big\} = \pipoisson(\yy)
\end{equation}
Since $1/x$ is continuous, we may appeal to the continuous mapping theorem to achieve consistency:
\begin{equation}
    \lim_{M\rightarrow \infty} \frac{1}{ \frac{1}{M}\sum_{m=1}^M \mathbbm{1}\Big\{\yy \in Y_T^{(m)}\Big\}} =  \frac{1}{  \lim_{M\rightarrow \infty} \frac{1}{M}\sum_{m=1}^M \mathbbm{1}\Big\{\yy \in Y_T^{(m)}\Big\}} = \frac{1}{\pipoisson(\yy)}
\end{equation}
We can compute the asymptotic variance by the delta rule:
\begin{subequations}
\begin{align}
    \asvar\left[\frac{1}{\pipoissonmc(\yy)}\right] &= \frac{1}{M} \frac{ \mathbb{V}\left[ \pipoissonis(\yy)\right]}{\pipoisson(\yy)^4} \quad \quad\quad\quad\text{\color{gray} (apply the delta rule)} \\ &= \frac{1}{M} \frac{  \pipoisson(\yy) - \pipoisson(\yy)^2}{\pipoisson(\yy)^4} \quad \text{\color{gray} (plugging in the variance computed above)}\\
    &= \frac{1}{M} \left(\frac{1}{\pipoisson(\yy)^3}- \frac{1}{\pipoisson(\yy)^2} \right)
\end{align}
\end{subequations}
\end{proof}

\propthree*
\begin{proof}\label{app:propthree}
\noindent i) We first prove that the estimator of the inclusion probabilities is unbiased through
the following manipulation:
\allowdisplaybreaks
\begin{subequations}\begin{align}
\E\left[ \pipoissonis(\yy)\right] &= \E \Bigg[\frac{1}{M}\sum_{m=1}^M  \prod_{t=1}^T \piQ(\yy_{\leq t} \mid \hY{t-1}^{(m)}) \Bigg]\\
&= \sum_{\hY{1}, \ldots,\hY{T}} \tildepoisson(\hY{1}, \ldots, \hY{\nmax}) \prod_{t=1}^{T} \piQ(\yy_{\leq t} \mid \hY{t-1})\\
&= \sum_{\hY{1}, \ldots,\hY{T}} \tildepoisson(\hY{1}, \ldots, \hY{\nmax}) \frac{\poisson(\hY{1}, \ldots, \hY{\nmax})}{\poisson(\hY{1}, \ldots, \hY{\nmax})} \prod_{t=1}^{T}  \piQ(\yy_{\leq t} \mid \hY{t-1})  \\
&= \sum_{\hY{1}, \ldots,\hY{T}} \poisson(\hY{1}, \ldots, \hY{\nmax}) \frac{\tildepoisson(\hY{1}, \ldots, \hY{\nmax})}{\tildepoisson(\hY{1}, \ldots, \hY{\nmax})} \quad\quad\quad \text{\color{gray} {(\cref{lem:decomposition})}} \\ 
& = \sum_{\hY{1}, \ldots,\hY{T}} \poisson(\hY{1}, \ldots, \hY{\nmax})
\\ 
& = \sum_{Y_1, \ldots, Y_T} \poisson(Y_1, \ldots, Y_{\nmax})\,\mathbbm{1}\Big\{ \yy \in Y_T\Big\} \quad\quad\quad \text{\color{gray} {(definition of $\hY{T}$)}}
\\ 
&= \pipoisson(\yy)
\end{align}\end{subequations}

\noindent ii) To show consistency, we appeal to the strong law of large number and the continuous mapping theorem.
By the strong law of large numbers, we have that 
\begin{equation}
    \lim_{M \rightarrow \infty} \frac{1}{M}\sum_{m=1}^M  \prod_{t=1}^T \piQ(\yy_{\leq t} \mid \hY{t-1}^{(m)}) = \pipoisson(\yy)
\end{equation}
Since $1/x$ is continuous, we have
\begin{subequations}
\begin{align}
     \lim_{M \rightarrow \infty} \frac{1}{\frac{1}{M}\sum_{m=1}^M  \prod_{t=1}^T \piQ(\yy_{\leq t}  \mid \hY{t-1}^{(m)})} &= \frac{1}{\lim_{M \rightarrow \infty} \frac{1}{M}\sum_{m=1}^M  \prod_{t=1}^T \piQ(\yy_{\leq t}  \mid \hY{t-1}^{(m)})} \nonumber \\
     &= \frac{1}{\pipoisson(\yy)}
\end{align}\end{subequations}
which shows consistency. 
Now, we derive a bound on the asymptotic variance of the inverse inclusion probabilities.
First, suppose that 
\begin{equation}
\frac{\prod_{t=1}^T \piQ(\yy_{\leq t}\mid \hY{t-1})}{\pipoisson(\yy)} \leq r,\quad \forall\,\hY{1}, \ldots, \hY{T}
\end{equation}
We start with the variance of importance sampling. This is a standard result \cite{monte-carlo}. 
Then we proceed with algebraic manipulation integrating the assumption above:
\begin{subequations}
\begin{align}
    \sum_{\hY{1}, \ldots,\hY{T}} &\frac{ \mathbbm{1}\Big\{\yy \in \hY{T} \Big\}^2\poisson(\hY{1}, \ldots,\hY{T})^2}{\tildepoisson(\hY{1}, \ldots,\hY{T})} - \pipoisson(\yy)^2 \\
    &=
     \sum_{\hY{1}, \ldots,\hY{T}} \poisson(\hY{1}, \ldots,\hY{T}) \prod_{t=1}^T \piQ(\yy_{\leq t}\mid \hY{t-1}) - \pipoisson(\yy)^2 \\
    &\leq
     \sum_{\hY{1}, \ldots,\hY{T}} \poisson(\hY{1}, \ldots,\hY{T}) \pipoisson(\yy) r - \pipoisson(\yy)^2 \\
     &=
     \pipoisson(\yy) \pipoisson(\yy) r - \pipoisson(\yy)^2 \\
      &=
     \pipoisson(\yy)^2 r - \pipoisson(\yy)^2 \\
     &=(r - 1) \pipoisson(\yy)^2
\end{align}
\end{subequations}
We can compute the asymptotic variance by the delta rule:
\begin{subequations}
\begin{align}
    \asvar\left[\frac{1}{\pipoissonis(\yy)}\right] &= \frac{1}{M} \frac{ \mathbb{V}\left[ \pipoissonis(\yy)\right]}{\pipoisson(\yy)^4} \quad\quad \text{\color{gray} (apply the delta rule)} \\ 
    &\leq \frac{1}{M} \frac{(r - 1) \pipoisson(\yy)^2}{\pipoisson(\yy)^4} \quad \text{\color{gray} (plugging in the above bound)} \\
     &= \frac{1}{M} \frac{(r - 1)}{\pipoisson(\yy)^2}
\end{align}
which proves the result.
\end{subequations}
\end{proof}
\newpage
\section{Pseudocode}\label{app:pseudocode}

\begin{algorithm}[H]
\textbf{Input:} $K$: Size of subset \\
\hspace*{2.7em} $w_1, \ldots, w_N$: weights for each element of the base set \\
\textbf{Output:} $W$: elementary symmetric polynomials of $w_1, \ldots, w_N$; $Z = W_{N, K}$  \\
\begin{algorithmic}[1]
\State $W \gets \boldsymbol{0}^{(K+1) \times (N+1)}$
\State $W_{0,:} = 0;\, W_{:,0} = 1$
\For{$n = 1, \dots, N$}
\For{$k = 1, \dots, K$}
    \State $W_{n,k} \gets W_{n-1,k} + w_n W_{n-1, k-1}$
\EndFor
\EndFor
\State \Return $W$
\end{algorithmic}
\caption{Dynamic program algorithm for $\Z$}
\label{alg:Z}
\end{algorithm}

\newcommand{\dW}[0]{\adj{W}}
\newcommand{\dw}[0]{\adj{w}}

\begin{algorithm}[H]
\textbf{Input:} $K$: Size of subset \\
\hspace*{2.7em} $w_1, \ldots, w_N$: weights for each element of the base set 
\begin{algorithmic}[1]
\State Run \cref{alg:Z} to compute $W$
\LinesComment{The code below was derived by manually apply algorithmic differentiation \cite{bucker2006automatic} to \cref{alg:Z}.}
\State $\dW \gets \boldsymbol{0}^{(K+1) \times (N+1)}$ \rightcomment{Initialize adjoints}
\State $\dw \gets \boldsymbol{0}^{N}$
\State $\dW_{N,K} = 1$  \rightcomment{Initialize output value to $1$}
\For{$n = N, \ldots, 1$}
\For{$k = K, \ldots, 1$}
    \State $\dw_{n} \pluseq \dW_{n,k} \, W_{n-1,k-1}$
    \State $\dW_{n-1,k-1} \pluseq \dW_{n,k} \, w_n$
    \State $\dW_{n-1,k} \pluseq \dW_{n,k}$
\EndFor
\EndFor
\LinesComment{Apply \cref{eq:incl-from-grad}}
\State $\pi \gets \boldsymbol{0}^N$
\For{$n = 1 \ldots N$}
\State $\pi_n \gets \frac{w_n}{\Z}\dw_n$
\EndFor
\Return $\pi$
\end{algorithmic}
\caption{Dynamic program for calculating inclusion probabilities $\picp$}
\label{alg:inc}
\end{algorithm}

\newpage
\section{Experimental Setup and Additional Results}\label{app:results}
We use a Transformer-based model trained according to \citet{ott-etal-2018-scaling} on the WMT'14 English-French dataset.\footnote{available at \url{http://statmt.org/wmt14/translation-task.html}} We use the pre-trained model checkpoints made available by fairseq.\footnote{ \url{https://github.com/pytorch/fairseq/tree/master/examples/translation}}
Data preprocessing steps, model hyperparameters and baseline performances can be found in the original work and on the fairseq website. All evaluations are performed on the \texttt{wmt14.v2.en-fr.newstest2014} version of the \texttt{newstest} data set. We show additional results using the setup in \cref{sec:experiments} in \cref{fig:entropy-experiment,fig:bleu-experiment-unnormalized,fig:bleu-experiment-not-nucleus}. We provide an empirical runtime analysis in \cref{fig:runtime}. \cref{table:baseline-variance} shows the variance of baseline estimator value for the three sentences used in RMSE experiments.

\begin{figure*}[h]
    \centering
    \includegraphics[width=\textwidth]{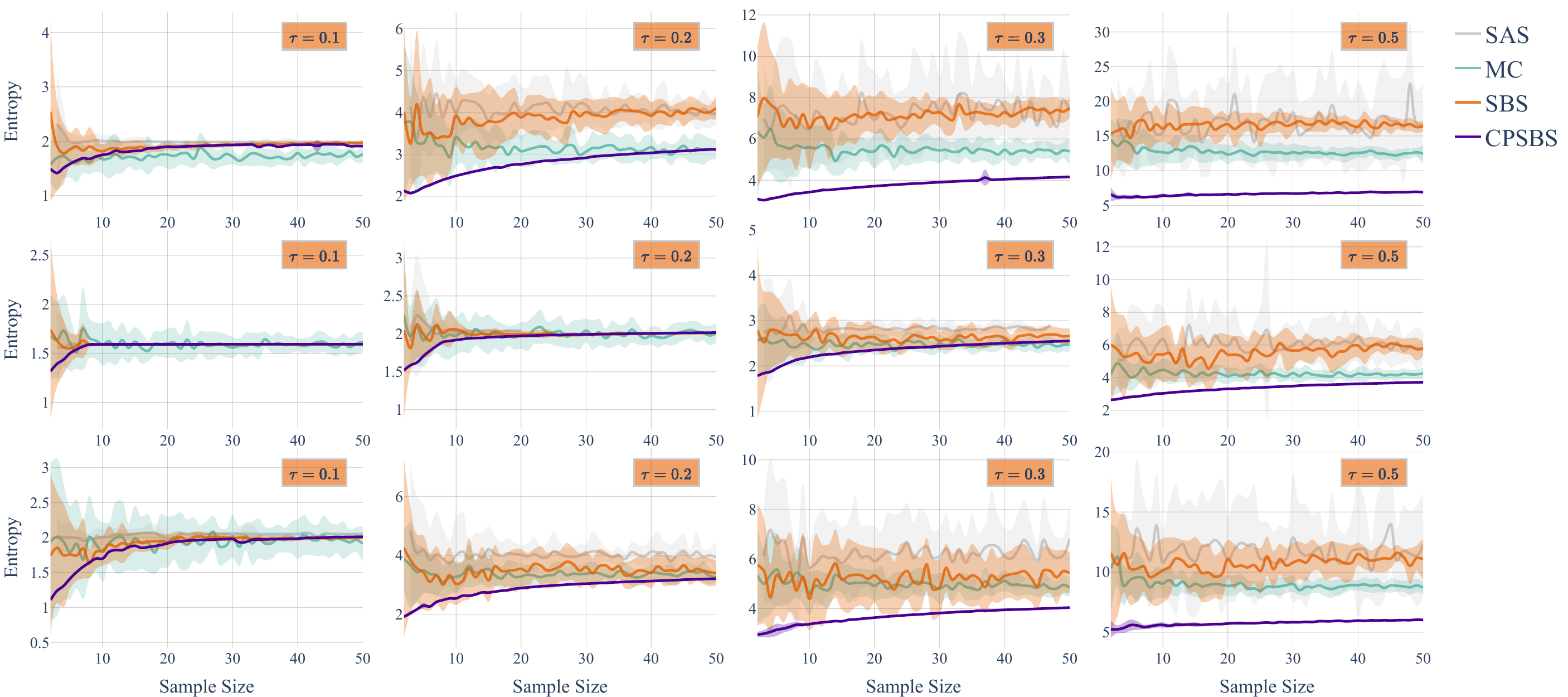}
     \caption{Entropy estimates for three different sentences using estimators for respective decoding methods. $\tau$ indicates scaling temperature. Values are chosen to mimic \cite{kool2019stochastic}.}
        \label{fig:entropy-experiment}
\end{figure*}

\begin{figure*}[ht!]
    \centering
    \includegraphics[width=\textwidth]{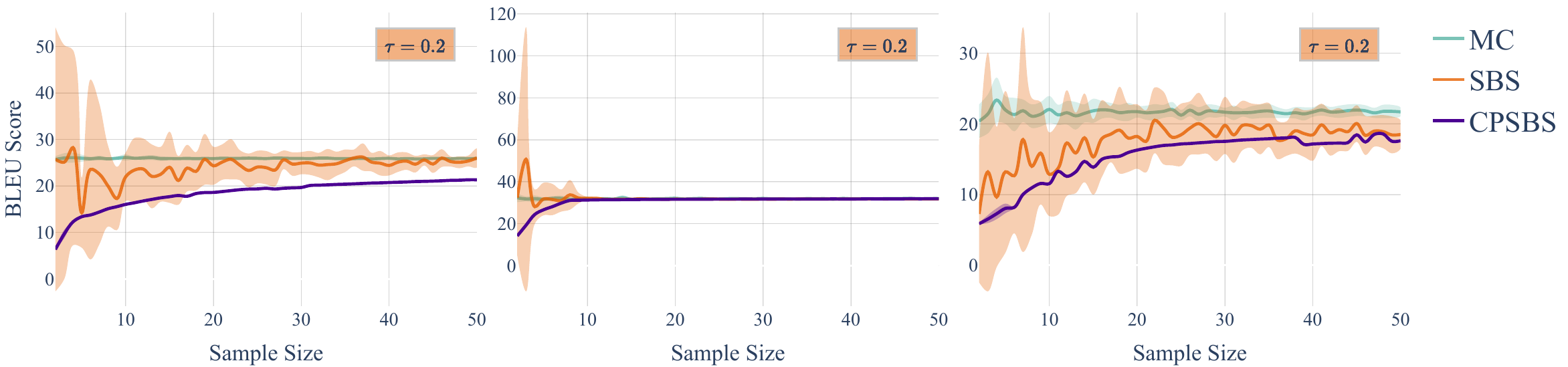}
     \caption{\bleu score estimates using unnormalized versions of SBS and CPSBS estimators. }
        \label{fig:bleu-experiment-unnormalized}
\end{figure*}

\begin{figure*}[ht!]
    \centering
    \includegraphics[width=\textwidth]{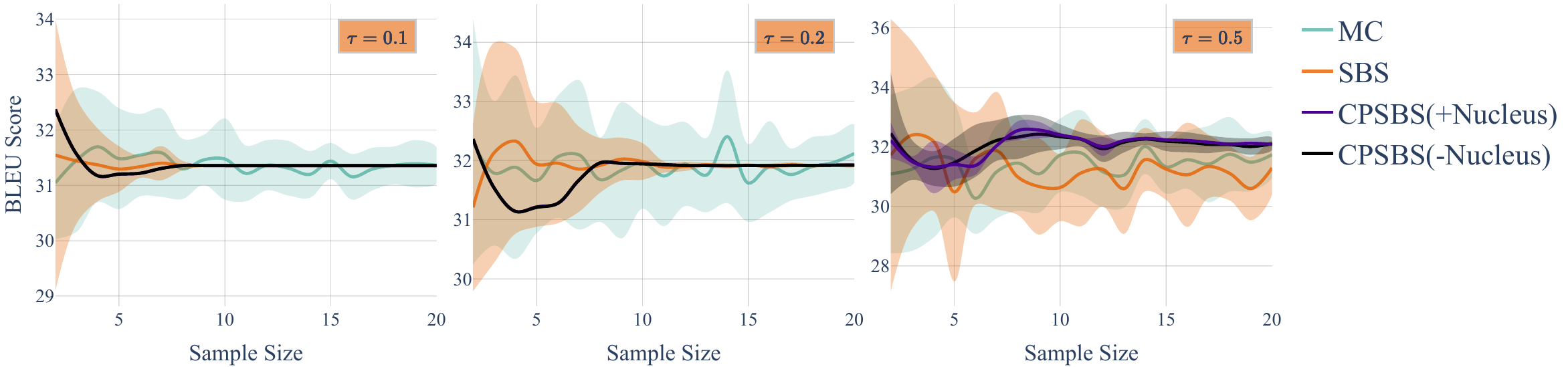}
     \caption{\bleu score estimates for CPSBS both with and without truncation of the sampling distribution. We see that our estimator with truncation provides virtually the same results.}
        \label{fig:bleu-experiment-not-nucleus}
\end{figure*}

\begin{figure*}[ht!]
    \centering
    \includegraphics[width=\textwidth]{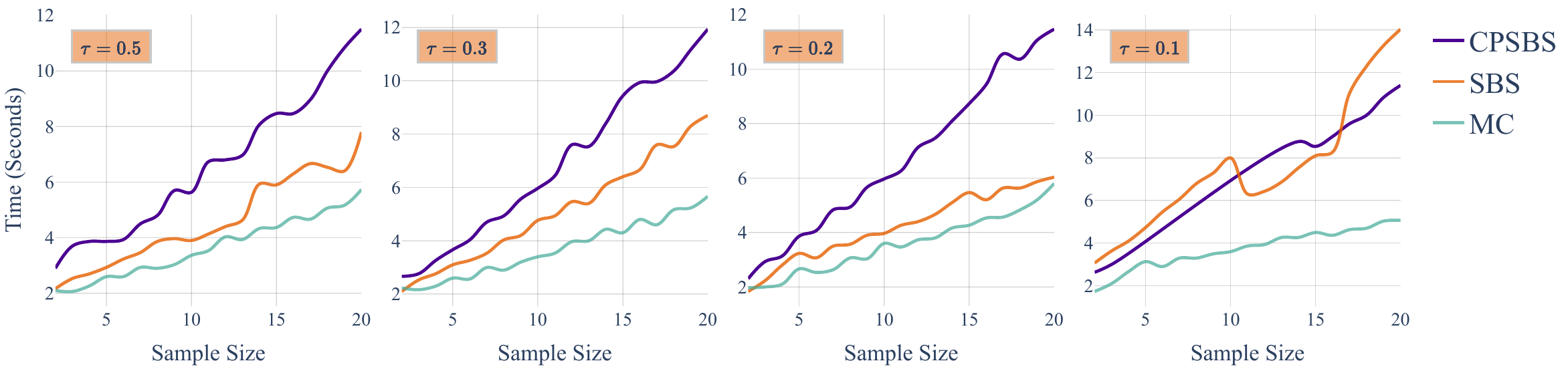}
     \caption{A comparison between decoding time of different sampling methods. The y-axis shows the average decoding time of the three sentences as before. The x-axis shows the number of samples taken for each sentence. All methods are tested on CPU.}
        \label{fig:runtime}
\end{figure*}

\begin{table*}[h]
\centering
\small
\ra{1.3}
\begin{tabular}{@{}rrrrrcrrrr@{}}\toprule
& \multicolumn{4}{c}{\bleu Estimator} & \phantom{abc} & \multicolumn{4}{c}{Entropy Estimator}\\
\cmidrule{2-10}
& \multicolumn{4}{c}{Temperature} & \phantom{abc} & \multicolumn{4}{c}{Temperature}\\
 \cmidrule{2-5}  \cmidrule{7-10} 
& 0.1 & 0.2 & 0.3 & 0.5 &&  0.1 & 0.2 & 0.3 & 0.5 \\ \midrule
Sentence\# 1500&
0.00& 0.00 & 0.03 & 0.08 &&
0.00 & 0.01 & 0.07 & 0.10 \\
Sentence\# 2000&
 0.01	& 0.04 & 0.04 & 0.08 &&
0.00 & 0.00 & 0.00 & 0.02 \\
Sentence\# 2500&
 0.07	& 0.09 & 0.25 & 0.84 &&
0.00 & 0.01 & 0.03 & 0.04 \\
\bottomrule
\end{tabular}
\caption{Variance of baseline estimator (MC for $k=200$ in 50 iterations) for the three sentences. \label{table:baseline-variance}}
\end{table*}

\onecolumn
\inLongVersion{
\section{Proposal Distribution Algorithm}\label{alg:proposal}
\begin{algorithm}[h!]
\textbf{Input:} $K$: number of (desired) samples \\
\hspace*{2.7em} $\nmax$: maximum sequence length \\
\hspace*{2.7em} $\ptheta(\cdot)$: model
\begin{algorithmic}[1]
\Function{sample\_cp}{$\Lambda, k, \mathrm{ind}$}
\State $k = k - 1$
\State $\Lambda \gets  \Lambda[0:\mathrm{ind}] + \Lambda[\mathrm{ind}+1:]+ \Lambda[\mathrm{ind}]$ \Comment{reorder such that weight pertaining to $\mathrm{ind}$ is last}
\State $S \gets [\,\mathrm{ind} \,]$
\State $E \gets \textsc{inclusion\_probs}(\Lambda, k)$
\For{$i \in \mathrm{length}(\Lambda) + 1, \dots, 1$}
    \State  $u \gets \textsc{random\_uniform}()$
    \If{$u < \Lambda[i] * E[k-1][i-1] / E[k][i]$}
        \State $i' \gets i \, \mathbf{if}\, i < \mathrm{ind}\, \mathbf{else}\, i+1$
        \State $S.\mathrm{add}(i')$
        \State $k = k - 1$
        \If{$k == 0$}
            \State \textbf{break}
        \EndIf
    \EndIf
\EndFor
\State \Return $S$
\EndFunction
\end{algorithmic}
\caption{Modified sampling scheme for proposal distribution $Q^{\textsc{CP}}_k$}
\label{alg:proposal}
\end{algorithm}

\subsection{Notes on how to set the weights}

Notice in \cref{eq:inclusion-prob} that if we set $w(\yy)=p_i/(1-p_i)$ that the inclusion probabilities are not equal to $p_i$.  How should we set the parameters $w(\yy)$ so that the inclusion probabilities are proportional to $p_i$?  Unfortunately, there is no closed-form solution.  But there is a numerical solution.

\newcommand{\vw}[0]{\boldsymbol{w}}
Formally, we would like to find $w(\yy)$ such that
$$
J(\vw) \defeq \sum_{i \in \calB} (\pi^{\mathrm{CPS}}_i(\vw) - p^\star_i)^2
$$
is minimized.  Surprisingly, this optimization problem is convex and has a unique solution.

\cite{aires1999algorithms} adjusts the weights to match inclusion probabilities using a numerical root-finding algorithm.

\url{https://www.overleaf.com/5144282174mzzmgkyqcyxd}}

\end{document}